\def\year{2018}\relax
\documentclass[letterpaper]{article} 
\usepackage{aaai18}  
\usepackage{times}  
\usepackage{helvet}  
\usepackage{courier}  
\usepackage{url}  
\usepackage{graphicx}  
\usepackage{amssymb}
\usepackage{amsmath}
\usepackage[defblank]{paralist}

\usepackage{relsize}
\usepackage{xifthen}
\usepackage{xparse}

\usepackage[amsmath,thmmarks]{ntheorem}
\theoremseparator{.}

\theorembodyfont{\itshape}
 \theoremsymbol{\ensuremath{\qed}}
\newtheorem{theorem}{Theorem}

\newtheorem{lemma}{Lemma}

\theorembodyfont{\normalfont}
\theoremsymbol{\ensuremath{\blacksquare}}
\newtheorem{definition}{Definition}

\theoremsymbol{\ensuremath{\square}}
\newtheorem{example}{Example}

\theorembodyfont{\normalfont}
\theoremstyle{nonumberplain}
\theoremsymbol{\ensuremath{\square}}
\theoremheaderfont{\normalfont\itshape}
\newtheorem{proof}{Proof}

\usepackage[textsize=scriptsize,textwidth=1.7cm]{todonotes}
\usepackage{cooltooltips}
\usepackage{etoolbox}
\makeatletter
\patchcmd{\@addmarginpar}{\ifodd\c@page}{\ifodd\c@page\@tempcnta\m@ne}{}{}
\makeatother
\reversemarginpar

\usepackage{enumitem}
\usepackage{times}
\usepackage{helvet}
\usepackage{courier}

\usepackage{latexsym}
\usepackage{amsmath}
\usepackage{amssymb}

\usepackage{caption}
\usepackage{subcaption}

\usepackage{multicol}
\usepackage{comment}
\usepackage{textcomp}
\usepackage{xspace}
\usepackage{soul}\setuldepth{x}
\usepackage{url}

\usepackage{tikz,pgf}
\usetikzlibrary{arrows,decorations,backgrounds,matrix,automata,shapes,shadows,trees,shapes,patterns,calc,fit,positioning}
\usetikzlibrary{petri,arrows,snakes,backgrounds,matrix,automata,shapes,shadows,patterns,fit,calc}

\pgfdeclarelayer{background}
\pgfdeclarelayer{foreground}
\pgfsetlayers{background,main,foreground}
\usepackage{pgfplots}
\usetikzlibrary{pgfplots.groupplots}

\tikzstyle{every picture}=[->,>=stealth',shorten >=1pt,auto,node distance=1.3cm,semithick]  

\tikzstyle{place}=[circle,thick,draw=black,minimum size=4mm]
\tikzstyle{invisible place}=[place,draw=none,fill=none]
\tikzstyle{transition}=[rectangle,inner ysep=1,thick,draw=black!75,fill=black!10,minimum size=2mm,,minimum width=4mm]
\tikzstyle{itransition}=[transition,draw=none,fill=none]
  \tikzstyle{every label}=[black]

\begin{document}

\title{Situation Calculus for Synthesis of Manufacturing Controllers}
\author{
Giuseppe De Giacomo\\
Sapienza Universit\`{a} di Roma, Italy\\
\url{degiacomo@dis.uniroma1.it}
\And 
Paolo Felli\\
Free University of Bozen-Bolzano, Italy\\
\url{paolo.felli@unibz.it} 
\AND
Brian Logan\\
University of Nottingham, UK\\
\url{bsl@cs.nott.ac.uk} 
\And 
Fabio Patrizi\\
Sapienza Universit\`{a} di Roma, Italy\\
\url{patrizi@dis.uniroma1.it}
\And 
Sebastian Sardina\\
RMIT University, Australia\\
\url{sebastian.sardina@rmit.edu.au} 
}
\maketitle


\newcommand{\mathname}[1]{\ensuremath{\text{\textit{#1}}}}
\newcommand{\textmath}[1]{\mathname{#1}}
\newcommand{\propername}[1]{\text{\textsf{\small #1}}\xspace}

\newcommand{\actionfont}[1]{\text{\textsc{#1}}}
\newcommand{\fluentfont}[1]{\mathname{#1}}

\newcommand{\astar}{\propername{A*}}
\newcommand{\JPS}{\propername{JPS}}
\newcommand{\HPA}{\propername{HPA*}}

\newcommand{\invAStar}{\propername{Inv-A*}}
\newcommand{\invJPS}{\propername{Inv-JPS}}
\newcommand{\invHPA}{\propername{Inv-HPA*}}


\newcommand{\compl}[1]{\overline{#1}}
\newcommand{\quotes}[1]{{\lq\lq #1\rq\rq}}
\newcommand{\Set}[1]{\left\{#1\right\}}
\newcommand{\bigmid}{\Big|}
\newcommand{\card}[1]{|{#1}|}                     
\newcommand{\Card}[1]{\left| #1\right|}
\newcommand{\cards}[1]{\sharp #1}
\newcommand{\sub}[1]{[#1]}
\newcommand{\config}[1]{\tuple{#1}}	
\newcommand{\set}[1]{\{#1\}}
\newcommand{\tup}[1]{\langle #1\rangle} 
\newcommand{\underoverset}[3]{\underset{#1}{\overset{#2}{#3}}}

\newcommand{\isdef}{\hbox{$\stackrel{\mbox{\tiny def}}{=}$}}

\newcommand{\powerset}{\mathbb{P}}
\newcommand{\NatN}{\Math{\mathbb{N}_0}} 
\newcommand{\Nat}{\Math{\mathbb{N}}}
\newcommand{\mgu}{\modesf{mgu}}
\newcommand{\complexsub}{\modesf{cplex}}


\newcommand{\formulas}[1]{\mathbb{F}_#1}
\newcommand{\propositions}{\mathbb{P}}
\newcommand{\literals}[1]{\mathbb{L}_#1}
\newcommand{\FFF}{\mathbb{F}}
\newcommand{\PPP}{\mathbb{P}}
\newcommand{\NNN}{\mathbb{N}}
\newcommand{\LLL}{\mathbb{L}}
\newcommand{\true}{\mathsf{TRUE}}
\newcommand{\false}{\mathsf{FALSE}}
\newcommand{\True}{\mathtt{True}}
\newcommand{\False}{\mathtt{False}}
\newcommand{\T}{\mathtt{T}}
\newcommand{\F}{\mathtt{F}}
\newcommand{\sat}{\mathcal{SAT}}
\newcommand{\no}[1]{\neg #1}
\newcommand{\prop}{\mathname{Prop}} 
\newcommand{\Pre}{\textit{Pre}}
\newcommand{\Eff}{\textit{Eff}}
\newcommand{\eff}{\mathname{eff}} 
\newcommand{\Cond}{\textit{Cond}}

\newcommand{\result}{\mbox{\emph{next}}}
\newcommand{\resultExec}{\mbox{\emph{result}}}

\newcommand{\EventCandidates}{\mathname{EventCand}\xspace}
\newcommand{\ExtendNet}{\mathname{ExtNet}\xspace}
\newcommand{\ON}{\mathname{ON}\xspace}

\newcommand{\A}{\mathcal{A}} 
\newcommand{\B}{\mathcal{B}}
\newcommand{\C}{\mathcal{C}} 
\newcommand{\D}{\mathcal{D}}
\newcommand{\E}{\mathcal{E}} 
\newcommand{\G}{\mathcal{G}} 
\renewcommand{\H}{\mathcal{H}}
\newcommand{\I}{\mathcal{I}} 
\newcommand{\J}{\mathcal{J}}
\newcommand{\K}{\mathcal{K}} 
\renewcommand{\L}{\mathcal{L}}
\newcommand{\M}{\mathcal{M}} 
\newcommand{\N}{\mathcal{N}}
\renewcommand{\O}{\mathcal{O}} 
\renewcommand{\P}{\mathcal{P}}
\newcommand{\Q}{\mathcal{Q}} 
\newcommand{\R}{\mathcal{R}}
\renewcommand{\S}{\mathcal{S}} 
\newcommand{\U}{\mathcal{U}} 
\newcommand{\V}{\mathcal{V}}
\newcommand{\W}{\mathcal{W}} 
\newcommand{\X}{\mathcal{X}}
\newcommand{\Y}{\mathcal{Y}} 
\newcommand{\Z}{\mathcal{Z}}
\newcommand{\BE}{{\mathcal{B}, \mathcal{E}}}
\newcommand{\DB}{{\mathcal{DB}}}
\newcommand{\DS}{{\mathcal{DS}}}
\newcommand{\Int}{{\mathcal{I}}}
\newcommand{\Flts}{{\mathcal{F}}}

\newcommand{\adom}{{adom}}

\newcommand{\Lsc}{\mathcal{L}_{\mathname{sitcalc}}}
\newcommand{\nullc}{\mathname{null}}
\newcommand{\Next}{\mathname{next}}
\newcommand{\Null}{\mathname{null}}
\newcommand{\Succ}{\mathname{succ}}
\newcommand{\Insert}{\mathname{Insert}}
\newcommand{\InsertNullB}{\mathname{InsertNull2}}
\newcommand{\DeleteB}{\mathname{Delete2}}
\newcommand{\doc}{\mathname{do}}

\newcommand{\ap}{\mathname{ap}}
\newcommand{\una}{\mathname{una}}
\newcommand{\ssa}{\mathname{ssa}}

\newcommand{\concat}{\cdot}
\newcommand{\goto}[1]{\stackrel{#1}{\longrightarrow}}
\newcommand{\tuple}[1]{\tup{#1}}            

\newcommand{\myi}{\emph{(i)}\xspace}
\newcommand{\myii}{\emph{(ii)}\xspace}
\newcommand{\myiii}{\emph{(iii)}\xspace}
\newcommand{\myiv}{\emph{(iv)}\xspace}
\newcommand{\myv}{\emph{(v)}\xspace}
\newcommand{\myvi}{\emph{(vi)}\xspace}

\newcommand{\defterm}[1]{\ul{\textit{#1}}}	

\newcommand{\pnet}{\textsf{\small pnet}}
\newcommand{\readarcs}{\textsf{\small persprec}}

\newcommand{\liftRight}{\mathname{liftRight}}
\newcommand{\liftLeft}{\mathname{liftLeft}}

\newcommand{\rank}{\mathname{rank}}
\newcommand{\Unf}{\mathname{Unf}}
\newcommand{\df}{DF}
\newcommand{\dt}{DT}
\newcommand{\rf}{RF}
\newcommand{\rt}{RT}

\newcommand{\marking}[1]{\mbox{\textsf{\small Mark}}(#1)}

\newcommand{\ERVFLY}{\textsf{\small  ERV-Fly}}
\newcommand{\ERVFLYP}{\Math{\mbox{\textsf{\small ERV-Fly}}_{min}}}

\newcommand{\cst}[1]{#1}
\newcommand{\nat}{\mathbb{N}}

\newcommand{\prepl}[1]{{}^{\bullet}{#1}}
\newcommand{\postpl}[1]{{#1}^{\bullet}}

\newcommand{\Ctrl}{\textsc{Ctrl}}
\newcommand{\doingx}[1]{\Math{\textit{doing}_{#1}}}
\newcommand{\dox}[1]{\Math{\textit{do}_{#1}}}
\newcommand{\finx}[1]{\Math{\textit{fin}_{#1}}}
\newcommand{\start}[1]{\Math{\textit{start}_{#1}}}
\newcommand{\complete}[1]{\Math{\textit{complete}_{#1}}}


\newcommand{\CA}{\mathname{CA}}
\newcommand{\RA}{\mathname{RA}}
\newcommand{\FF}{\mathname{F}}
\newcommand{\RR}{\mathname{R}}
\newcommand{\GG}{\mathname{G}}
\newcommand{\Aset}{\mathname{A}}
\newcommand{\Pset}{\mathname{P}}
\newcommand{\Tset}{\mathname{T}}
\newcommand{\Vset}{\mathname{V}}

\newcommand{\Var}{\mathname{Var}}
\newcommand{\val}{\mathname{val}}
\newcommand{\expr}{\mathtt{expr}}
\newcommand{\net}{\mathcal{N}}

\newcommand{\marked}[1]{\overline{#1}}

\newcommand{\enabled}[2]{#1 [#2\rangle}
\newcommand{\allreach}[1]{[#1\rangle}
\newcommand{\fire}[4]{\enabled{#1}{#2}{#3}#4}
\newcommand{\guard}[1]{[\![#1]\!]}

\newcommand{\col}[1]{\ensuremath{\textsc{#1}}\xspace} 

\tikzstyle{place}=[circle,thick,draw=black,fill=white,minimum size=5mm,font=\fontsize{9}{144}\selectfont]
\tikzstyle{transition}=[rectangle,thick,draw=black,fill=gray!20,minimum size=5mm]
\tikzstyle{enabledtransition}=[rectangle,very thick,draw=green!75,fill=green!20,minimum size=5mm]
 \tikzstyle{container}=[rectangle,rounded corners,very thick,draw=black!75,fill=black!20,minimum height=7mm,minimum width=14mm]


\newcommand{\equip}{\taskfont{equip}}
\newcommand{\equipx}{\taskfont{equip\_x}}
\newcommand{\tinsert}{\taskfont{insert}}
\newcommand{\applyglue}{\taskfont{apply\_glue}}
\newcommand{\attach}{\taskfont{attach}}
\newcommand{\equipd}{\taskfont{eqp\_d}}
\newcommand{\equipr}{\taskfont{eqp\_r}}
\newcommand{\equipeng}{\taskfont{eqp\_eng}}
\newcommand{\equipg}{\taskfont{eqp\_g}}
\newcommand{\equipf}{\taskfont{eqp\_f}}
\newcommand{\equiph}{\taskfont{eqp\_h}}
\newcommand{\pressure}{\taskfont{app\_press}}
\newcommand{\pressureh}{\taskfont{apply\_press\_hollow}}
\newcommand{\pressuref}{\taskfont{apply\_press\_flat}}
\newcommand{\pressurex}{\taskfont{apply\_press\_x}}
\newcommand{\unequip}{\taskfont{unequip}}
\newcommand{\rivet}{\taskfont{rivet}}
\newcommand{\reload}{\taskfont{reload}}
\newcommand{\tremove}{\taskfont{remove}}
\newcommand{\vtest}{\taskfont{visual\_test}}
\newcommand{\ftest}{\taskfont{force\_test}}
\newcommand{\hold}{\taskfont{hold}}
\newcommand{\reaming}{\taskfont{reaming}}
\newcommand{\holdplace}{\taskfont{hold\_in\_place}}
\newcommand{\place}{\taskfont{place}}
\newcommand{\talign}{\taskfont{align}}
\newcommand{\drill}{\taskfont{drill}}
\newcommand{\machinedrill}{\taskfont{machine\_drill}}
\newcommand{\robotdrill}{\taskfont{robot\_drill}}
\newcommand{\cellstore}{\taskfont{out\_cell}}
\newcommand{\store}{\taskfont{store}}
\newcommand{\position}{\taskfont{position}}
\newcommand{\cellloads}{\taskfont{in\_cell}}
\newcommand{\loads}{\taskfont{load}}
\newcommand{\separate}{\taskfont{separate}}
\newcommand{\forcepull}{\taskfont{force\_pull}}
\newcommand{\loadsmaterial}{\taskfont{load\_material}}
\newcommand{\loadh}{\taskfont{load\_hv}}
\newcommand{\enter}{\taskfont{enter}}
\newcommand{\exit}{\taskfont{exit}}
\newcommand{\screw}{\taskfont{screw}}
\newcommand{\inspect}{\taskfont{inspect}}
\newcommand{\clamp}{\taskfont{clamp}}
\newcommand{\mountbit}{\taskfont{mount\_bit}}
\newcommand{\operatemachine}{\taskfont{operate\_machine}}
\newcommand{\engrave}{\taskfont{engrave}}
\newcommand{\fasten}{\taskfont{fasten}}
\newcommand{\unfasten}{\taskfont{unfasten}}
\newcommand{\rotate}{\taskfont{rotate}}
\newcommand{\spray}{\taskfont{spray}}
\newcommand{\ProdRes}{\R}

\newcommand{\videocapture}{\taskfont{video\_capture}}

\newcommand{\nname}[1]{\small{\textbf{#1}}}

\newcommand{\nop}{\taskfont{nop}}

\newcommand{\taskin}{in}
\newcommand{\taskout}{out}

\newcommand{\Golog}{\propername{Golog}}
\newcommand{\GologSpeak}{\propername{GologSpeak}}
\newcommand{\DGolog}{\propername{DGolog}}
\newcommand{\sGolog}{\propername{sGolog}}
\newcommand{\ConGolog}{\propername{ConGolog}}
\newcommand{\IndiGolog}{\propername{IndiGolog}}
\newcommand{\LeGolog}{\propername{LeGolog}}
\newcommand{\DTGolog}{\propername{DTGolog}}
\newcommand{\Prolog}{\propername{Prolog}}
\newcommand{\AgentSpeak}{\propername{AgentSpeak}}
\newcommand{\JASON}{\propername{Jason}}
\newcommand{\CANMINUS}{\propername{\CAN$^{\A}$}}
\newcommand{\CANMINUST}{\propernametiny{Can$^{\cal C}$}}
\newcommand{\CAN}{\propername{CAN}}
\newcommand{\CANT}{\propernametiny{Can}}
\newcommand{\CANPLAN}{\propername{CANPlan}}
\newcommand{\CANPLANT}{\propernametiny{CanPlan}}
\newcommand{\CANPLANII}{\propername{CanPlan2}}
\newcommand{\CANPLANOR}{\propername{Can(Plan)}}
\newcommand{\CANGOAL}{\propername{CanGoal}}
\newcommand{\JACK}{\propername{Jack}}
\newcommand{\JACKTM}{\propername{Jack\texttrademark}}
\newcommand{\JAM}{\propername{JAM}}
\newcommand{\PRS}{\propername{PRS}}
\newcommand{\SPARK}{\propername{SPARK}}
\newcommand{\RAP}{\propername{Rap}}
\newcommand{\dMARS}{\propername{dMARS}}
\newcommand{\TAPL}{\propername{3APL}}
\newcommand{\DAPL}{\propername{2APL}}
\newcommand{\GOALBDI}{\propername{GOAL}}
\newcommand{\JSHOP}{\propername{JSHOP}}
\newcommand{\JSHOPII}{\propername{JSHOP2}}
\newcommand{\ASHOP}{\propername{A-SHOP}}
\newcommand{\SHOP}{\propername{SHOP}}
\newcommand{\SHOPII}{\propername{SHOP2}}
\newcommand{\ACT}{\propername{ACT}}
\newcommand{\SIPEII}{\propername{SIPE-2}}
\newcommand{\OPLANII}{\propername{O-PLAN2}}
\newcommand{\Retsina}{\propername{Retsina}}
\newcommand{\IPEM}{\propername{IPEM}}
\newcommand{\SAGE}{\propername{Sage}}
\newcommand{\DECAF}{\propername{Decaf}}
\newcommand{\PROPICE}{\propername{Propice-Plan}}
\newcommand{\CYPRESS}{\propername{Cypress}}
\newcommand{\CPEF}{\propername{CPEF}}
\newcommand{\JADEX}{\propername{JADEX}}
\newcommand{\IMPACT}{\propername{IMPACT}}
\newcommand{\PDT}{\propername{PDT}}


\newcommand{\mif}{\mbox{\bf if}}
\newcommand{\mwhile}{\mbox{\bf while}}
\newcommand{\mreturn}{\mbox{\bf return}}
\newcommand{\mthen}{\mbox{\bf then}}
\newcommand{\melse}{\mbox{\bf else}}
\newcommand{\mdo}{\mbox{\bf do}}
\newcommand{\mnoOp}{\mbox{\bf $noOp$}}
\newcommand{\mproc}{\mbox{\bf proc}}
\newcommand{\mend}{\mbox{\bf end}}
\newcommand{\mendproc}{\mbox{\bf endProc}}
\newcommand{\mendif}{\mbox{\bf endIf}}
\newcommand{\mendwhile}{\mbox{\bf endWhile}}
\newcommand{\mendfor}{\mbox{\bf endFor}}
\newcommand{\mfor}{\mbox{\bf for}}
\def\prparallel{\mathrel{\rangle\!\rangle}}
\def\supparallel{\mathord{|\!|}}
\newcommand{\ndet}{\mid}
\newcommand{\choice}[1]{\pi{#1}.}
\newcommand{\conc}{\mbox{$\parallel$}}
\newcommand{\tconc}{\mbox{$\space|||\space$}}
\newcommand{\pconc}{\mbox{$\prparallel$}}
\newcommand{\search}{\mbox{$\Sigma$}}
\newcommand{\searchO}{\mbox{$\Sigma_o$}}
\newcommand{\searchOM}{\mbox{$\Sigma_o^M$}}
\newcommand{\searchCR}{\mbox{$\Sigma_{cr}$}}
\newcommand{\searchR}{\mbox{$\Sigma_{r}$}}
\newcommand{\searchM}{\mbox{$\Sigma^M$}}
\newcommand{\searchC}{\mbox{$\Sigma_c$}}
\newcommand{\searchCM}{\mbox{$\searchC^M$}}
\newcommand{\searchCB}{\mbox{$\Sigma_{cb}$}}
\newcommand{\searchD}{\mbox{$\Delta$}}
\newcommand{\searchE}{\mbox{$\Delta_e$}}
\newcommand{\searchEM}{\mbox{$\Delta_e^M$}}
\newcommand{\searchL}{\mbox{$\Delta_l$}}
\newcommand{\searchER}{\mbox{$\Delta_r$}}
\newcommand{\searchERM}{\mbox{$\Delta_r^M$}}
\newcommand{\mnt}{\mbox{$mnt$}}

\newcommand{\Know}{\mbox{\bf Know}}
\newcommand{\KWhether}{\mbox{\bf KWhether}}
\newcommand{\Kref}{\mbox{\bf KRef}}
\newcommand{\nows}{{\hbox{\small\sf now}}}
\newcommand{\now}{{\mbox{\sf now}}}

\newcommand{\Sensed}{\textmath{Sensed}}
\newcommand{\hend}{\textmath{end}}
\newcommand{\Trans}{\textmath{Trans}}
\newcommand{\Final}{\textmath{Final}}
\newcommand{\Poss}{\textmath{Poss}}
\newcommand{\Transobs}{\textmath{TransObs}}

\DeclareDocumentCommand{\ext}{ m O{\TS} O{v} }{%
        (#1)^{#2}_{#3}
}

\newcommand{\BAT}[1][]{\mathcal{D}_{{#1}}\xspace}
\newcommand{\BATT}{\BAT[\pedixT]}
\newcommand{\BATS}{\BAT[\pedixS]}

\newcommand{\confT}{\textsc{c}_{\pedixT}}
\newcommand{\confS}{\textsc{c}_{\pedixS}}
\newcommand{\Confs}[1]{\textsc{C}_{#1}}

\newcommand{\fluent}{\mathname{f}}
\newcommand{\Fluents}[1][]{\mathcal{F}_{#1}}

\renewcommand{\true}{\textsf{\smaller[1]true}}
\newcommand{\taskfont}[1]{\smaller[.8]\actionfont{#1}} 
\newcommand{\const}[1]{\texttt{\smaller[1]#1}}

\renewcommand{\do}{\mathname{do}}
\newcommand{\Do}{\mathname{Do}}
\newcommand{\pa}{\conc}

\newcommand{\progTinit}{\progT^0}
\newcommand{\progSinit}{\progS^0}

\newcommand{\pedixS}{\textsc{s}}
\newcommand{\pedixT}{\textsc{t}}

\newcommand{\progT}{\prog_\pedixT}
\newcommand{\progS}{\prog_\pedixS}

\newcommand{\prog}{\delta}

\newcommand{\sit}{s}
\newcommand{\sitT}{s_\pedixT}
\newcommand{\sitS}{s_\pedixS}
\newcommand{\SitT}{S_\pedixT}
\newcommand{\SitS}{S_\pedixS}

\newcommand{\xx}{\vec{x}}
\newcommand{\xxT}[1][]{\vec{x}_{\pedixT#1}}
\newcommand{\xxS}[1][]{\vec{x}_{\pedixS#1}}

\newcommand{\yy}{\vec{y}}
\newcommand{\zz}{\vec{z}_{\pedixS}}

\renewcommand{\vec}[1]{\mathbf{#1}}

\newcommand{\turn}{\vartheta}
\newcommand{\turnT}{\pedixT}
\newcommand{\turnS}{\pedixS}

\newcommand{\PhiGA}{\Phi_{GA}}
\newcommand{\phiWinOLD}{{\Win}_{\Phi_{Sim}}}
\DeclareDocumentCommand{\phiWin}{ o }{%
        \IfNoValueTF{#1}%
            {\Win(\Phi_{Sim})}%
            {\Win^{#1}(\Phi_{Sim})}%
}
\DeclareDocumentCommand{\phiWinPrime}{ o }{%
        \IfNoValueTF{#1}%
            {\Win(\Phi'_{Sim})}%
            {\Win^{#1}(\Phi'_{Sim})}%
}

\newcommand{\gfp}{\textsc{gfp}\xspace}
\newcommand{\lfp}{\textsc{lfp}\xspace}

\newcommand{\Res}[1]{\R_{#1}}

\DeclareDocumentCommand{\pred}{ m o o }{%
        \mathname{#1}%
        \IfNoValueTF{#2}%
            {}%
        {    
        	{(\mathname{\smaller[.8]$#2$}}%
        	\IfNoValueTF{#3}%
            	{)}%
            	{,#3)}
        }%
}
\DeclareDocumentCommand{\action}{ m o }{%
        \actionfont{#1}%
        \IfNoValueTF{#2}%
            {}%
            {(\mathname{\smaller[.8]$#2$})}%
}


\newcommand{\actTermV}{a} 
\newcommand{\ActTermV}{\boldsymbol a} 

\newcommand{\actTerm}{\alpha} 
\newcommand{\ActTerm}{\boldsymbol\alpha} 
\newcommand{\Act}{A} 
\newcommand{\act}{a} 

\newcommand{\TS}{\mathcal{T}}
\newcommand{\finite}[1]{{#1}^\mathname{f}}

\newcommand{\ConfsT}{C_{\pedixT}}
\newcommand{\ConfsS}{C_{\pedixS}}

\newcommand{\TransObs}{\mathname{TransObs}}
\newcommand{\FinalObs}{\mathname{FinalObs}}
\newcommand{\Obs}{\mathname{Obs}}

\newcommand{\Win}{\textit{Win}}

\newcommand{\strat}{\varsigma}
\newcommand{\fstrat}{\finite{\varsigma}}
\newcommand{\contr}{\rho}

\begin{abstract}
Manufacturing is transitioning from a mass production model to a manufacturing as a service model in which manufacturing facilities `bid' to produce products. To decide whether to bid for a complex, previously unseen product, a manufacturing facility must be able to synthesize, `on the fly', a process plan controller that delegates abstract manufacturing tasks in the supplied process recipe to the appropriate manufacturing resources, e.g., CNC machines, robots etc. Previous work in applying AI behaviour composition to synthesize process plan controllers has considered only finite state ad-hoc representations.  Here, we study the problem in the relational setting of the Situation Calculus. By taking advantage of recent work on abstraction in the Situation Calculus, process recipes and available resources are represented by ConGolog programs over, respectively, an abstract and a concrete action theory.  This allows us to capture the problem in a formal, general framework, and show decidability for the case of bounded action theories. We also provide techniques for actually synthesizing the controller.
\end{abstract}

\section{Introduction}\label{sec:introduction}


Manufacturing is transitioning from a mass production model to a service model in which manufacturing facilities `bid' to produce products. 
In contrast to mass production, where large volumes of known products are produced at a time, in manufacturing as a service, the products to be manufactured are not known in advance, batch sizes are often small, and a facility may produce products for several customers at the same time \cite{TSB:12a,Rhodes:15a}. This trend towards rapid provisioning of manufacturing resources,  e.g., CNC machines and robots, with minimal management effort or service provider interaction has been termed `cloud manufacturing' \cite{Xu:12a,Lu//:14a}.

To determine if a novel product can be manufactured, abstract manufacturing tasks in the \emph{process recipe} specifying how a product is manufactured must be matched against the available manufacturing resources so as to produce a \emph{process plan} detailing the low-level tasks to be executed and their order, the manufacturing resources to be used, and how materials and parts move between resources \cite{Groover:07a}. Control software---called \emph{process plan controller}---that  delegates each operation in the plan to the appropriate manufacturing resources is then synthesized. 
In mass production, process planning is carried out by manufacturing engineers and is largely a manual process. However this is uneconomic for the small batch sizes typical of the manufacturing as a service model, and the time required to produce a plan is too great to allow facilities to bid for products in real time. To fully realize the manufacturing as a service vision, manufacturing facilities must be able to \emph{automatically synthesize process plan controllers} for novel products `on the fly'.

There has recently been efforts to apply AI behavior composition \cite{DPS-AIJ13} to the synthesis of process plan controllers, e.g., \cite{de-Silva//:16a,Felli//:16a,Felli//:17a,De-Giacomo//:18a}. 
However, this work suffers from two important limitations. 
First, the approaches are restricted to \emph{finite} state representations. 
While adequate for some manufacturing tasks and resources, the resulting discretization is unwieldy and less natural, for example, when representing and reasoning about the potentially infinite number of basic parts, each with a unique bar code or RFID tag. 
Secondly, existing approaches to modeling abstract manufacturing tasks and the interactions between resources result in somewhat ad-hoc and inflexible formalisms. For example, the state (and its dynamics) of a given part (e.g., painted, defective, etc.) or of a shared resource (e.g., a conveyor belt) ought to be encoded into operational representations, such as transition systems or automata. Clearly, more declarative representations would be desirable. 

In this paper, we address both issues by adopting the Situation Calculus~\cite{Reiter:BOOK01}, a knowledge representation formalism to reason about action and change. 
We represent process recipes and available resources as high-level \ConGolog programs executing over, respectively, an abstract and a concrete action theory.  
This yields a principled, formal, and declarative representation of the manufacturing setting. 
We then define, by means of a suitable simulation relation, what it means to realize the a product process recipe on the available resources.
Finally, by leveraging recent work on abstraction, we show how to effectively check the existence of a process plan controller (and how to compute it) for the case of bounded action theories.


\section{Situation Calculus}


The situation calculus \cite{McCarthy:MI69-SitCalc,Reiter:BOOK01} is a logical language for representing and reasoning about dynamically changing worlds in which all changes are the result of \textit{actions}. 
We assume to have a finite number of action types, each of which
takes a tuple of objects as arguments.
For example, $\action{\drill}(part,dmtr,speed,x,y,z)$ represents the (simple) action of drilling a hole of a certain diameter at a certain speed in a given part. 
In the manufacturing domain we are concerned with operations that may occur simultaneously. We therefore adopt the \emph{concurrent non-temporal} variant of the situation calculus~\cite[Chapter 7]{Reiter:BOOK01}, where a concurrent or compound action $\ActTermV$ is a, possibly infinite, set of of simple actions $a$ that execute simultaneously.
For example, $\set{\action{\rotate}(part,speed),\action{\spray}(part)}$ represents the joint execution of rotating a part at a given speed while spraying it.
Situations denote possible sequences of concurrent actions: the constant $S_0$ denotes the initial situation, on which we assume to have complete information, and the situation resulting from executing a concurrent action $\ActTermV$ in a situation $s$ is represented as situation term $\do(\ActTermV,s)$. 
Predicates whose value varies from situation to situation are called \emph{fluents}, and they take arguments of sort object plus a situation term as their last argument. For example, $\pred{painted(part,s)}$ may denote that a part is painted in situation $s$.

A \emph{basic action theory (BAT)}~\cite{PirriReiter:JACM99-ContributionsSC,Reiter:BOOK01} is a collection of axioms $\D$ describing the \emph{preconditions} and \emph{effects} (and non-effects) of actions on fluents.
A special predicate $\Poss(a,s)$ is used to state that the simple action $a$ is executable in situation $s$, and, for each simple action type, a precondition axiom is given to specify when the action can be legally performed. Such a predicate is extended to compound actions  $\Poss(\ActTermV,s)$, typically by requiring that each atomic action in $\ActTermV$ is possible, i.e., $\Poss(a,s)$, although one can further restrict $\Poss$ when needed. We also assume that $\Poss(\set{a},s)\equiv\Poss(a,s)$. 
A \emph{successor state axiom} is used to specify how each fluent changes as the result of executing (simple or) concurrent actions in the domain. Successor state axioms thus encode the causal laws of the domain being modelled, by encoding the effects of actions. 
Figures \ref{fig:resources_bats} and \ref{fig:sys_bat} list examples of precondition and successor state axioms for a manufacturing setting. 




\section{A Variant of ConGolog for Manufacturing} 
\emph{High-level programs} are used to specify complex processes in the domain.
%
We specify programs in (recursion-free) \ConGolog\ \cite{DeGiacomoLL:AIJ00-ConGolog}:
\begin{quote}
$\ActTermV$           	\hfill  compound action\\
$\phi?$            	\hfill test for a condition\\
$\delta_1;\delta_2$        
					\hfill  sequence\\
$\delta_1 \ndet \delta_2$   
					\hfill  nondeterministic branch\\
$\pi x.\delta$     	\hfill  nondeterministic choice of argument\\
$\delta^*$         	\hfill  nondeterministic iteration\\
$\mif\ \phi\ \mthen\ \delta_1\ \melse\ \delta_2\ \mendif$
                	\hfill  conditional\\
$\mwhile\ \phi\ \mdo\ \delta\ \mendwhile$
                	\hfill  while loop\\
$\delta_1 \conc \delta_2$  
					\hfill interleaved concurrency 
\end{quote}
where $\ActTermV$ is a compound action (instead of atomic action as in the original paper) and $\phi$ is situation-suppressed formula, i.e., a formula in the language with all situation arguments in fluents suppressed. 
We denote by $\phi[\sit]$ the situation calculus formula obtained from $\phi$ by restoring the situation argument $\sit$ into all fluents in $\phi$. 
We require that the variable $x$ in programs of the form $\pi x.\delta$ range over objects, and occurs in some action term in $\delta$, i.e., $\pi x.\delta$ acts as a construct for the nondeterministic choice of action parameters. 

The semantics of \ConGolog\ is specified in terms of single-steps, using the following two predicates \cite{DeGiacomoLL:AIJ00-ConGolog}: 
\begin{compactitem}
  \item $\Final(\prog,\sit)$: program $\delta$ may  terminate in situation $\sit$; and 
  \item $\Trans(\prog,\sit,\prog',\sit')$: one step of program $\prog$ in situation $\sit$ may lead to situation $\sit'$ with $\prog'$ remaining to be executed.
\end{compactitem}
The definitions of $\Trans$ and $\Final$ for the standard \ConGolog constructs are given by: 
\[
 \begin{array}{l}
\Final(\ActTermV,s) \equiv \False \\
\Final(\phi?,s) \equiv \phi[s] \\
\Final(\delta_1; \delta_2,s) \equiv 
        \Final(\delta_1,s) \land \Final(\delta_2,s)\\
\Final(\delta_1 | \delta_2,s) \equiv 
	\Final(\delta_1,s) \lor \Final(\delta_2,s) \\
\Final(\pi x.\delta, s) \equiv
	\exists x. \Final(\delta,s)\\
\Final(\delta^*, s) \equiv \True\\
\Final(\delta_1 \conc \delta_2,s) \equiv 
       \Final(\delta_1,s) \land \Final(\delta_2,s)
 \end{array}
\]
\[
 \begin{array}{l}
\Trans(\ActTermV,s,\delta',s') \equiv \\
\qquad\qquad
	s'=do(\ActTermV,s) \land \Poss(\ActTermV,s) \land \delta'=\True? \\
\Trans(\phi?,s,\delta',s') \equiv \False \\
\Trans(\delta_1; \delta_2,s,\delta',s') \equiv 
	\Trans(\delta_1,s,\delta_1',s') \land \delta'=\delta_1';\delta_2 \lor{}\\
\qquad\qquad   
		\Final(\delta_1,s) \land \Trans(\delta_2,s,\delta',s') \\
\Trans(\delta_1 \ndet \delta_2,s,\delta',s') \equiv{} \\
\qquad\qquad   
	\Trans(\delta_1,s,\delta',s') \lor \Trans(\delta_2,s,\delta',s') \\
\Trans(\pi x.\delta, s, \delta',s') \equiv 
	\exists x. \Trans(\delta,s,\delta',s')\\
\Trans(\delta^*, s, \delta',s') \equiv
	\Trans(\delta,s,\delta'',s')\land \delta'=\delta'';\delta^*\\
\Trans(\delta_1 \conc \delta_2,s,\delta',s') \equiv{} \\ 
\qquad\qquad
	\Trans(\delta_1,s,\delta_1',s')  \land \delta'=\delta_1'\conc\delta_2 \lor {}\\
\qquad\qquad
		\Trans(\delta_2,s,\delta_2',s')  \land \delta'=\delta_1\conc\delta_2'\\
 \end{array}
\]
%
Note that the conditional and while-loop constructs are  definable: $\mif\ \phi\ \mthen\ \delta_1\ \melse\ \delta_2\ \mendif = \phi?; \delta_1 | \lnot\phi?; \delta_2$ and $\mwhile\ \phi\ \mdo\ \delta\ \mendwhile = (\phi?; \delta)^*;\lnot \phi?$.

\smallskip
In the manufacturing setting, high-level programs are used to model the logic of manufacturing resources in ``isolation'', and synchronized concurrency is needed to represent the operation of two (or more) resources ``simultaneously''. For this reason, we introduce a new construct in \ConGolog, called \emph{synchronized concurrency}:

\begin{quote}
$\delta_1 \tconc \delta_2$  
					\hfill synchronized concurrency
\end{quote}
that represents the synchronized concurrent execution of programs $\delta_1$ and $\delta_2$: their next corresponding actions  
take place in the same next transition step.  Its semantics is defined as follows: 
\[
 \begin{array}{l}
\Trans(\prog_1 \tconc \prog_2,s,\prog',s')\equiv\\
\qquad\qquad
	[\Trans'(\prog_1,s,\prog'_1,s'_1) \land s'_1 = \do(\ActTermV_1,s) \land{}\\
\qquad\qquad	
	\Trans'(\prog_2,s,\prog'_2,s'_2) \land s'_2 = \do(\ActTermV_2,s) \land{}\\ 
 \qquad\qquad		
 	\Poss(\ActTermV_1 \cup \ActTermV_2 ,s) \land{}\\
\qquad\qquad		
	\prog' = (\prog'_1\tconc\prog'_2) \land s'=\do(\ActTermV_1 \cup \ActTermV_2,s)]
 \end{array}
\]

\noindent
where $\Trans'$ is equivalent to $\Trans$ except for the condition $\Trans(\ActTermV,s,\delta',s')$ which is now $\Trans'(\ActTermV,s,\delta',s') \equiv s'=do(\ActTermV,s) \land \delta'=\True?$, and for $\Trans'(\prog_1\tconc \prog_2,s,\delta',s')$ that is as above but without check of $\Poss$. 

The characterization of $\Final$ for synchronized concurrency is analogous to that of interleaved concurrency:
\[\Final(\prog_1 \tconc \prog_2,s)\equiv \Final(\prog_1,s)\land\Final(\prog_2,s).\]

Note the underlying assumption here is that a number of sub-systems, manufacturing resources in our setting, can legally perform a joint step if such joint step is deemed possible by the BAT.

\section{Modelling Manufacturing as a Service}

In this section we show how the manufacturing as a service setting can be captured by representing both process recipes and manufacturing systems  as \ConGolog programs.  

\subsection{Resource Programs}

We consider a \emph{manufacturing system}, or simply system, composed of $n$ of distinct manufacturing resources, each identified by an index $i\in[1,n]$. Each resource is associated to a BAT $\BAT[i]$ specifying the fluents, the actions that may be performed, their preconditions and effects. 
For convenience we assume that actions have the resource index $i$ (a constant) as their last argument. 

\begin{example}
Consider a manufacturing cell that performs operations on parts (the example is based on the cell described in \cite{Felli//:18a}). Parts have an ID, and parameters such as size, weight, material, etc. 
The cell consists of five resources. $\Res{1}$ is a robot that can perform different operations on parts within its operating envelope, by (autonomously) equipping with the appropriate end effector using the action $\action{\equip}[ee,\const{1}]$. By equipping a driller it can drill parts; by equipping a rivet gun it can apply rivets, etc. The drilling operation is modeled as the action $\action{\robotdrill}$ with arguments $part,bit,dmtr,speed,feed,x,y,z$ for the part ID, the drilling bit ID, the diameter, the spindle speed, the feed rate, and hole position. Fully specified actions are of the form $\action{\robotdrill}[\const{p},\const{bit1},\const{.7},\const{125},\const{.2},\const{123},\const{87},\const{12},\const{1}]$, with $\const{1}$ as last argument. Similarly for $\action{\rivet}[part,rivet\_type,x,y,z,\const{1}]$.

We model the other resources in a similar way. $\Res{2}$ is a fixture that can perform an action $\action{\holdplace}[part,force,\const{2}]$, with arguments for the part ID and the clamping force. 
Resource $\Res{3}$ is another robot that can move parts into and out of the cell from an external conveyor, position a part at a given location relative to another part, and, by equipping a flat or hollow end effector, apply pressure to a part that is being worked on by another resource (hollow for drilling or milling, flat for riveting). These operations correspond to the actions $\action{\cellloads}[part, weight, material, dimx, dimy, dimz,\const{3}]$, $\action{\cellstore}[part,code,\const{3}]$, $\action{\position}[part1,part2,x,y,z,\const{3}]$, $\action{\pressureh}[part,force,\const{3}]$, $\action{\pressuref}[part,force,\const{3}]$. Note that information about the weight, material and size of parts loaded into the cell is made available by passing this via arguments to $\action{\cellloads}$. 
Resource $\Res{4}$ is a upright drilling machine for drilling parts with high precision. Finally, $\Res{5}$ is a human operator, who operates $\Res{4}$ and who can also bring small parts into and out of the cell, and apply glue to parts with $\action{\applyglue}[part,glue\_type,\const{5}]$.
Parts are moved between resources by a part-handling system. For simplicity, we model this using additional actions $\action{\taskin}[part,i]$ and $\action{\taskout}[part,i]$ for each resource, denoting that a part is moved into or out of the work area of the resource, respectively. We also have a special action $\action{\nop}$ that specifies when a resource may remain idle. 

To denote that a part is currently in the work area of a resource $i$, that a hole has been drilled in a part by a resource, or that a part has a certain material, etc.\ we use situation dependent or independent fluents of the form $\pred{at}[part, i][s]$, $\pred{drilled}[part, hole, i]$, $\pred{material}[part, m]$, etc. 
In Figure~\ref{fig:resources_bats} we list some examples of precondition axioms for theories $\BAT[i]$. 
\end{example}

Given the BAT $\BAT[i]$ of a resource (which include the possible actions) we can describe all possible sequences of operations that the resource can execute (in isolation) as a \ConGolog program $\prog_i$ with BAT $\BAT[i]$.

\subsection{The Available Program}

The set of BATs $\set{\BAT[1],\ldots,\BAT[n]}$ is then compiled into a single BAT $\BATS$, for the entire system, in a semi-automated fashion, e.g., by taking into account knowledge about which resources are connected by the part-handling system, which resources can work on the same parts, etc. 
To be able to represent dynamic worlds that allow the concurrent execution of multiple actions, we consider \emph{compound action} terms of the form $\Act=\set{\act_1,\ldots,\act_k}$ with $k\leq n$, where each $\act_i$ is a basic action term. 
As shorthand, we denote by $\Act(\xx)$ the compound action $\Act$ with a vector $\xx$ of arguments (of the right size, and assuming a standard ordering of basic actions).
%
Moreover, to ensure that the resources involved can work on the same parts, we use a special situation-independent predicate $\pred{coopMatrix}[i,j]$ specifying that resource $i$ can cooperate with resource $j$. 

\begin{example}
The situation $\do(\{\action{\robotdrill}[\const{f}, \cdots,\const{1}]$, $\action{\holdplace}[\const{f},\cdots,\const{2}]\},\sit)$ 
results from the concurrent execution of two actions: $\Res{1}$ drilling a part with ID $\const{f}$ that is held by the fixture $\Res{2}$. 
%
Also, Figure~\ref{fig:sys_bat} shows a fragment of the resulting BAT $\BATS$ for the overall system. 
\end{example}

Similarly to the case of a single resource, assuming a set of $n$ BATs for each of the available resources, and the resulting BAT $\BATS$ for the entire system, we can  capture all the possible executions of the system  as a \ConGolog program. 

\begin{definition}[Available Program]
Given a set of $n$ resource programs $\prog_i$, $i\in[1,n]$, the resulting \emph{available program} is the \ConGolog program is 
$\progSinit := \prog_1 \tconc\cdots\tconc \prog_n$
\end{definition}

\begin{figure*}[h]
\small
$\Poss(\action{\robotdrill}[part,bit,dmtr,speed,feed,x,y,z,i],s) \equiv \pred{equipd}[\const{driller},i][s] \land \exists j.  \pred{at}[part,j][s] \land \pred{material}[part,m] \land \pred{ok}[bit,dmtr,m]$\\
$\Poss(\action{\machinedrill}[part,bit,dmtr,speed,feed,x,y,z,i],s) \equiv \pred{at}[part,i][s] \land \pred{material}[part,m] \land \pred{ok}[bit,dmtr,m] $\\
$\Poss(\action{\pressureh}[part,force,i],s) \equiv \pred{equipd}[\const{pressure\_hollow},i][s] \land \exists j.~  \pred{at}[part,j][s] $\\
$\Poss(\action{\pressuref}[part,force,i],s) \equiv \pred{equipd}[\const{pressure\_flat},i][s] \land \exists j.~  \pred{at}[part,j][s] $\\
$\Poss(\action{\hold}[part,force,i],s) \equiv \pred{at}[part,i][s] \land (\pred{equipd}[\const{grip\_s},i][s] \land  \pred{shape}[part,\const{squared}] \lor \pred{equipd}[\const{grip\_r},i][s] \land  \pred{shape}[part,\const{round}])$\\
$\Poss(\action{\rivet}[part,rivet\_type,x,y,z,i],s) \equiv \pred{equipd}[\const{rivet\_gun},i][s] \land \pred{hole}[part,x,y,z][s] \land \pred{compr}[i,\const{charged}][s] \land \exists j.~  \pred{at}[part,j][s] $

\begin{subfigure}[b]{0.55\textwidth}
$\Poss(\action{\equip}[ee,i],s) \equiv \pred{has\_effector}[i,ee][s] \land \neg\exists e.~ \pred{equipd}[e,i][s]$\\
$\Poss(\action{\loads}[part,weight,material,x,y,z,i],s) \equiv \pred{on\_site}[part][s]$\\
$\Poss(\action{\cellloads}[part,weight,material,x,y,z,i],s) \equiv \pred{on\_site}[part][s]$\\
$\Poss(\action{\operatemachine}[j,i],s) \equiv \pred{resource\_type}[j,\const{machine\_tool}]$\\
$\Poss(\action{\position}[part,p2,x,y,z,i],s) \equiv \pred{at}[part,i][s] \land \exists j.~\pred{at}[p2,j][s] $\\
$\Poss(\action{\applyglue}[part,glue\_type,i],s) \equiv \exists j.~  \pred{at}[part,j][s]$
\end{subfigure}%
\begin{subfigure}[b]{0.45\textwidth}
$\Poss(\action{\nop}[i],s) \equiv \True$\\
$\Poss(\action{\unequip}[ee,i],s) \equiv \pred{equipd}[ee,i][s]$\\
$\Poss(\action{\store}[part,code,i],s) \equiv \pred{at}[part,i][s]$\\
$\Poss(\action{\holdplace}[part,force,i],s) \equiv \pred{at}[part,i][s]$\\
$\Poss(\action{\taskin}[part,i],s) \equiv \exists j.~ \pred{at}[part,j][s] \land i\neq j$\\
$\Poss(\action{\taskout}[part,i],s) \equiv \pred{at}[part,i][s]$
\end{subfigure}
\caption{Examples of precondition axioms for resources.}
\label{fig:resources_bats}
\end{figure*}
\begin{figure*}[h]
\small 
$\Poss(\ActTermV \cup \action{\nop}(i),s)  \equiv \Poss(\ActTermV,s)$ \\
$\Poss(\set{\action{\taskin}[part,i], \action{\taskout}[part,j]},s) \equiv \Poss(\action{\taskin}[part,i],s) \land \Poss(\action{\taskout}[part,j],s) \land \pred{partHandling}[j,i] $ \\
$\Poss(\set{\act_1(p1,\ldots,i),\act_2(p2,\ldots,j)},s) \equiv \Poss(\act_1(p1,\ldots,i),s) \land \Poss(\act_2(p2,\ldots,j),s) \land \pred{p1=p2} \land \pred{coopMatrix}[i,j]  \cdots $ 

\medskip
%
$\pred{at}[p,i][\do(\ActTermV,s)] \equiv \action{\cellloads}[p,i]\in \ActTermV \lor \action{\taskin}[p,i]\in \ActTermV \lor \cdots$\\
$\pred{material}[part,mtrl][\do(\ActTermV,s)]\equiv \action{\cellloads}[part,weight,mtrl,x,y,z,i]\in \ActTermV \lor \cdots$\\ 
$\pred{equipd}[e,i][\do(\ActTermV,s)] \equiv  \action{\equip}[e,i]\in \ActTermV \lor \cdots $\\
$\pred{drilled}[hole(part,x,y,z),i][\do(\ActTermV,s)] \equiv \action{\robotdrill}[part,\cdots,x,y,z,i]\in \ActTermV \lor  \action{\machinedrill}[part, \cdots, x,y,z,i]\in \ActTermV \lor \cdots$

\caption{Above: example precondition axioms for theory $\BATS$. Below: examples of successor-state axioms, where fluents are affected only by compound actions corresponding to meaningful combinations of operations, i.e., those that are allowed by $\Poss$ in $\BAT[S]$. For brevity, only cases in which fluents are affected are shown. }
\label{fig:sys_bat}
\end{figure*}

\subsection{Target Program}

The product recipe specifying the possible way(s) in which a product can be manufactured is a \ConGolog program $\progTinit$ which we call the \emph{target program}. $\progTinit$ has its own BAT $\BATT$, which is distinct from $\BATS$ (for any system $S$). 
In the manufacturing as a service model, product recipes are \emph{resource independent} \cite{ANSI/ISA:10a}, i.e., specified using action terms $\A_T$ and fluents $\Fluents[T]$ understood throughout the manufacturing cloud rather than in the theory of any particular facility. 

\begin{example}
An example target program specified using the resource independent BAT $\BATT$ is shown in Figure~\ref{fig:target_prog}. Two parts denoted by $\const{b}$ and $\const{f}$ are loaded into the cell, then glue is applied to $\const{b}$ and it is placed on $\const{f}$, resulting in a composite part denoted by $\const{fb}$. The loading of $\const{b}$ and the drilling of $\const{f}$ can occur in any order, but glue must be applied to $\const{f}$ before $\const{b}$ is placed. If the resource used for drilling is not high-precision, a reaming operation is performed. Finally a rivet is applied and $\const{fb}$ is stored. 
\end{example}

\begin{figure}[t!]
$
 \begin{array}{l}
\action{\loads}[\const{f},\const{steel},\const{18},\const{810},\const{756},\const{345}] ~;\\
\action{\loads}[\const{b},\const{steel},\const{2},\const{312},\const{23},\const{20}] \conc \action{\drill}[\const{f},\const{.3},\const{200},\const{1},\const{123},\const{89},\const{21}] ~;\\
\action{\applyglue}[\const{b},\const{str\_adh}]  ~;~ \action{\place}[\const{b},\const{f},\const{fb},\const{7},\const{201},\const{140}] ~;\\
\textbf{if}~ \neg\textit{precision}(\pred{hole}[\const{f},\const{123},\const{89},\const{21}],\const{high}) ~\textbf{do}\\
\quad \action{\reaming}[\const{fb},\const{.3},\const{123},\const{89},\const{21}]\\
\action{\rivet}[\const{fb},\const{123},\const{89},\const{21}] ~;~ \action{\store}[\const{fb},\const{ok}]
\end{array}
$
    \caption{Example of target program $\progTinit$.}
    \label{fig:target_prog}
\end{figure}

To establish the \emph{manufacturability} of a product by a given system, we must establish \emph{mappings} between the resource-independent BAT $\BATT$ and the BAT $\BATS$ of the system. 
In practice, these mappings are computed for each manufacturing system, automatically or by hand, at the moment of joining the manufacturing cloud \cite{Felli//:18a}. 
Inspired by \cite{BanihashemiGL17}: 
\begin{itemize}
\item each $\Act(\xx)$ in $\progTinit$ is mapped to a (possibly complex) program $\prog_{\Act}(\xx)$ in $\BATS$, e.g., passing of parts through the part-handling system, equipping effectors etc.; 
\item \emph{some} fluents $\fluent\in \Fluents[T]$ correspond to formulas over the fluents in $\Fluents[S]$, i.e., to establish the value of $\fluent$ one needs to \emph{observe} the situation of the underlying theory $\BATS$ of $\progSinit$. Hence we say that these fluents model ``observations", and use of a special unary predicate $\Obs$ to distinguish them. 
\end{itemize}
%
This gives a set of mapping rules of the form:
\[
\Act(\xx) \leftrightarrow \prog_{\Act}(\xx) \qquad \fluent(\xx) \leftrightarrow \varphi_{\fluent}(\xx)
\]



For example, a rule that maps the resource-independent $\drill$ action in Figure~\ref{fig:target_prog} to a program specifying the possible ways in which a drilling operation can be executed in a specific system might be:

\smallskip
\noindent
$
 \begin{array}{l}
\set{\action{\drill}[part,dmtr,speed,x,y,z]} \leftrightarrow (~\A_1 \tconc \cdots \tconc \A_n~)^* ~;\\
\qquad \textbf{if}~ \pred{size}[part,\const{large}]   ~\textbf{then}~  \varphi_{\action{d}}^1 \mid \varphi_{\action{d}}^2 ~\textbf{else}~  \varphi_{\action{d}}^2 \qquad \text{with:}
\end{array}
$

\smallskip
\noindent
\begin{tabular}{@{}r@{}c@{}l@{}}
$\varphi_{\action{d}}^1$ & $=$ & $\pi~ \pred{feed},\pred{i},\pred{j},\pred{k}. (\action{\pressureh}[p,i] \tconc$\\ 
& & $\action{\holdplace}[part,\const{2k},j]$ $\tconc$\\
& & $\action{\robotdrill}[part,bit,dmtr,speed,feed,x,y,z,k])$\\
$\varphi_{\action{d}}^2$ & $=$ & $\pi~ \pred{feed},\pred{i},\pred{j}.(\action{\operatemachine}[i,j] \tconc$\\
& & $\action{\machinedrill}[p,bit,diam,speed,x,y,z,j])$
\end{tabular}

%
%

\smallskip
Crucially, a number of preliminary actions are required for these loading actions to be executable (e.g. equipping the right end effectors, clearing the working space, etc.) but these are not explicitly listed, as it is one of the objectives of the composition. Each $\A_i$ stands for $\choice{\xx}~ \act_{i,1}(\xx) \mid \cdots \mid \act_{i,q_i}(\xx)$: each resource can perform any of their actions. 
We can write similar mapping rules for fluents, e.g., specifying how the precision of a drilled holes is observed. 

In modelling a manufacturing domain, it is often natural to consider that the target and available programs are bounded \cite{DeGiacomoLPS16}. In practice, this means that the information of interest \emph{in each moment}, corresponding to the parts that are being manufactured, the possible operations executed, their possible parameters and the data produced, are not arbitrarily large but are bounded by a known bound. E.g., resources have bounded capacity, a product recipe consists of finitely many parts and requires finitely many operations. This assumption will be used to give a decidable technique to synthesize controllers.

\section{Orchestration via Simulation}

To define the conditions under which a target program $\progTinit$ can be realised by executing the available program $\progSinit$, we relate their execution. Extending the definition in \cite{SardinaG09} to our setting, we define the notion of \emph{simulation} between programs: 

\medskip
\noindent
$(\progT,\sitT)\preceq(\progS,\sitS) \supset$

$\Final(\progT,\sitT) \supset \Final(\progS,\sitS) ~\land$

$\bigwedge_{\fluent \in \Fluents[T]} \forall\xx. \fluent(\xx,\sitT) \equiv \varphi_{\fluent}(\xx,\sitS) ~\land$

$\forall\progT',\Act,\xx. ~\TransObs(\progT,\sitT,\sitS,\progT',\do (\Act(\xx),\sitT)) \supset$

$\quad\quad \exists \progS',\sitS'.  ~\Trans^*(\progS,\sitS,\progS',\sitS') ~\land\Do(\prog_{\Act}(\xx),\sitS,\sitS') ~\land$

$\quad\quad (\progT',\sitT')\preceq(\progS',\sitS')$

\medskip
\noindent
where $\TransObs(\progT,\sitT,\sitS,\progT',\sitT')$ iff $\Trans(\progT,\sitT,\progT',\sitT')$ and $\Final$ is substituted with $\FinalObs$, defined as:

\medskip
\noindent
\begin{tabular}{@{}l @{ }c@{ } l}
$\FinalObs(\progT,\sitT,\sitS)$ & $\equiv$ & $\Final(\progT,\sitS)$ if $\progT = \phi?$ with $Obs(\phi)$ \\
& & $\Final(\phi?,\sitT)$ otherwise.\\
\end{tabular}

\medskip
\noindent
Intuitively, we use the situation for theory $\BATS$ for testing the situation suppressed fluents in the theory $\BATT$ which correspond to observations of the underlying situations $\sitS$. 
Moreover, $\Do(\prog,\sit,\sit') := \exists \prog.~ \Trans^*(\prog,\sit,\prog',\sit') \land \Final(\prog',\sit')$ is used to establish that there exists a complete execution of the program $\prog$ from $\sitS$ to $\sitS'$.

The relation above specifies the following property: for every possible step from situation $\sitS$, in which the target program evolves from $\progT$ to $\progT'$ by executing $\Act(\xx)$, there exists an execution of the concurrent program in situation $\sitS$, from $\progS$ to $\progS'$ (and through a complete execution $\prog_{\Act}(\xx)$), for which the same property holds. Also, whenever the target program can terminate, also the available program can. Through $\FinalObs$, we allow the target program to assess the values of fluents on the situation $\sitS$: the process recipes can specify conditions to be checked by observing the system. 
E.g., the value of $\pred{precision}$ in the program in Figure~\ref{fig:target_prog} must be observed after the drilling operation. 

Moreover, note that we could not simply replace each action $\Act(\xx)$ in the target program $\progTinit$ by its corresponding program $\prog_{\Act}(\xx)$ and then apply known approaches such as that of \cite{SardinaG09}: to satisfy the simulation requirement it is enough to find at least one way in which $\prog_{\Act}$ can be executed so that the simulation property is maintained, whereas a syntactical substitution we would require that all such evolutions must be possible in the system. 

Essentially, the simulation captures the fact that $\progTinit$ can implement the execution of $\progTinit$, subject to the mapping rules. 

\begin{definition}[Realizability]
The target program $\progTinit$  is \emph{realizable} by the available program 
$\progSinit$ if $(\progTinit,S^0_T)\preceq(\progSinit,S^0_S)$.
\end{definition}
When $\progTinit$ is realized by $\progSinit$ then, at every step, given a possible ground action $\Act(\xx)$ selected by $\progTinit$, one can select the corresponding program $\prog_{\Act}(\xx)$, 
execute it, and then return the control to the target program for the next action selection. Notice, however, that 
the execution of $\prog_{\Act}(\xx)$ is not deterministic, as \ConGolog programs include in general choices of 
arguments and nondeterministic branching. Nonetheless, the existence of the simulation guarantees that this is possible, but it does not detail how. Similarly to \cite{SardinaG09}, we assume to have total control on the interpreter executing the available concurrent program $\progSinit$, whose nondeterminism is `angelic', and  define here the notion of \emph{controller}: the unit responsible for orchestrating the system, hence the available resources, at each step. Intuitively, this requires to consider any possible execution of $\progTinit$ and $\progSinit$, as commented above, which are infinite. 

First, in order to \emph{isolate} the source of such infiniteness into the program \emph{data} only, for each program (target and available) we separate the assignments of pick variables to objects in the domain from the control flow of the programs, namely their \emph{program counter}. This is the approach of \cite{DeGiacomoLPS16}, which we adapt here to our framework. Hence, we equivalently represent a program $\prog^0$ as the couple $\tup{\prog,\xx}$, where $\prog$ merely denotes its current program counter,  and $\xx=\tup{x_1,\ldots,x_k}$ is a tuple of object terms so that each $x_i\in\Delta$ is the current value of $i$-th pick variable of $\prog^0$. We call $\xx$ the (current) \emph{environment}. 
Importantly, this is merely a syntactic manipulation: as showed in \cite{DeGiacomoLPS16}, we can reconstruct the original program $\prog^0$ by replacing the free pick variables of $\prog$ by those object terms to which variables $\xx$ are assigned. This is denoted by writing $\prog[\xx]$.  
Nonetheless, assuming programs without recursion, this simple technique allows one to obtain a \emph{finite} set of possible program counters for a given program, which we define next (the possible environments remain infinite). 

\begin{definition}[Syntactic closure of a program]
Given a program $\prog^0$, it is the set $\Gamma_{\prog^0}$ inductively defined as follows: $(1)$ $\prog^0 \in \Gamma_{\prog^0}$; $(2)$ if $\prog_1;\prog_2\in\Gamma_{\prog^0}$ and $\prog'_1\in \Gamma_{\prog_1}$ then $\prog'_1 ; \prog_2\in \Gamma_{\prog^0}$ and $\Gamma_{\prog_2}\subseteq \Gamma_{\prog^0}$; $(3)$ if $\prog_1\mid\prog_2\in \Gamma_{\prog^0}$ then $\Gamma_{\prog_1},\Gamma_{\prog_2}\subseteq\Gamma_{\prog^0}$; $(4)$ if $\choice{z}\prog\in\Gamma_{\prog^0}$ then $\Gamma_{\prog}\subseteq\Gamma_{\prog^0}$; $(5)$ if $\prog^*\in \Gamma_{\prog^0}$ then $\prog;\prog^*\in \Gamma_{\prog^0}$; $(6)$ if $\prog_1\conc\prog_2\in \Gamma_{\prog^0}$ and $\prog'_1\in \Gamma_{\prog_1}$ and $\prog'_2\in \Gamma_{\prog_2}$ then $\prog'_1\conc\prog'_2\in \Gamma_{\prog^0}$; $(7)$ if $\prog_1\tconc\prog_2\in \Gamma_{\prog^0}$ and $\prog'_1\in \Gamma_{\prog_1}$ and $\prog'_2\in \Gamma_{\prog_2}$ then $\prog'_1\tconc\prog'_2\in \Gamma_{\prog^0}$.
\end{definition}

Denoting the finite set of all possible environments of a program $\prog^0$ as $\Delta^k$, so that $k$ is the number of its pick variables, we call a triple $\tup{\prog,\xx,\sit}\in \Gamma_{\prog^0} \times \Delta^k \times S$ a (complete) \emph{configuration} of $\prog^0$.
Denoting the set of possible configurations as $\Confs{\delta^0}$, we can finally define our notion of controller which, intuitively, given the current  configurations $\confT$ and $\confS$ for the target and system programs, and a new configuration for the target, selects a sequence of configurations for the system so that the simulation relation is recovered.   

\begin{definition}[Controller]
Given a target program $\progTinit$  realizable by an available program $\progSinit$,
a \emph{controller} for $\progSinit$ that realizes $\progTinit$ is a function 
$\contr:\Confs{\progTinit}\times\Confs{\progSinit} \times \Confs{\progTinit} \rightarrow \Confs{\progSinit}^*$
s.t.:
\begin{itemize}
	\item $\contr(\tup{\progT,\xxT,\sitT},\tup{\progS,\xxS,\sitS},\tup{\progT',\xxT',\sitT'})$ is defined whenever $(\progT[\xxT],\sitT)\preceq(\progS[\xxS],\sitS)$
		and there exist $\Act,\xx$ s.t.~$\TransObs(\progT[\xxT],\sitT,\sitS,\progT'[\xxT'],\do(\Act(\xx),\sitT))$;
	\item if $\contr(\tup{\progT,\xxT,\sitT},\tup{\progS,\xxS,\sitS},\tup{\progT',\xxT',\sitT'})$ returns the sequence $\tup{\prog^0,\xx^0,\sit^0}\ldots\tup{\prog^m,\xx^m,\sit^m}$ then:
		\begin{itemize}
			\item for $i\in[0,m-1]$, $\Trans(\prog^i[\xx^i],\sit^i,\prog^{i+1}[\xx^{i+1}],\sit^{i+1})$,   
			with $\prog^0=\progS$, $\xxS^0=\xxS$ and $\sit^0=\sitS$, namely the sequence is executable in the available system; and
			\item $\tup{\progT'[\xxT'],\sitT'}\preceq\tup{\prog^m[\xx^m],\sit^m}$, i.e., the simulation between the resulting programs is preserved. Note that, by the definition of simulation, this implies that the sequence of system configurations returned by the controller correspond to a complete execution of $\prog_{\Act}(\xx)$.
		\end{itemize}
\end{itemize}
\label{def:controller}
\end{definition}

\section{Controller Synthesis}

To check whether a simulation exists and, if so, build a controller, we resort to model checking
for a variant of the (modal) $\mu$-calculus  in~\cite{CalvaneseDMP18},
interpreted over \emph{game arenas} (GA), i.e., special  
(labelled) transitions systems (TS) capturing turn-based game rules. 
We show that when such systems are state-bounded, computing winning strategies becomes decidable. 

\subsection{Model checking over game arenas}
For a set of fluents $\Fluents$ and an object domain $\Delta$, 
we denote by $\Int^{\Fluents}_\Delta$ the set of all (standard) FO 
interpretations $\tup{\Delta,\cdot^{\I(q)}}$ over $\Delta$ of the fluents 
in $\Fluents$.
 
\begin{definition}[Game arena]
Let $\Fluents$ be a set of fluents 
including the special $0$-ary fluents (i.e., propositions) $\pred{turnS}$ and $\pred{turnT}$,
and $\Delta$ an object domain.
A \emph{game arena} over $\Fluents$ and $\Delta$ is a tuple
$\TS = \tup{\Delta,Q,q_0,\rightarrow,\I}$, where: 
\begin{compactitem}
	\item $\Delta$ is the object domain;
	\item $Q$ is the set of GA states;
	\item $q_0\in Q$ is the GA initial state;
	\item $\rightarrow\subseteq Q\times Q$ is the GA transition relation;
	\item $\I:Q\rightarrow \Int^{\Fluents}_\Delta$ is a labeling function, associating to each state $q\in Q$ an interpretation $\I(q)=\tup{\Delta,\cdot^{\I(q)}}\in \Int^{\Fluents}_\Delta$,
		s.t.~exactly one among $\pred{turnS}$ and $\pred{turnT}$ is true.
\end{compactitem}
\label{def:TS}
\end{definition}
$\TS$ represents the moves available to two players, 
Target and System,  in a game, but not the game's goal. 
The arena is turn-based: Target and System can move in states where, 
respectively, $\pred{turnT}$ and $\pred{turnS}$ hold. 
Turns are not strictly alternating. 
Wlog, we assume that in $q_0$ it is Target's  turn.

Goals are expressed through $\mu$-calculus formulas. 
The language we use, called $\mu\L_{c}$ ($c$ indicates 
that we use only closed FO formulas), is:
$$\Phi := \phi \mid \neg \Phi \mid \Phi_1 \land \Phi_2 \mid \langle-\rangle \Phi \mid Z \mid \mu Z.\Phi \mid \nu Z.\Phi$$
\noindent
were $\phi$ is a FO sentence with predicates from $\Fluents$ and constants from $\Delta$
(under unique-name assumption, we can safely use objects from $\Delta$ as constants); 
the modal operator $\langle-\rangle \Phi$ denotes the existence of a transition from the 
current state to a state where $\Phi$ holds; we use the abbreviation $[-]\Phi$ for 
$\neg \langle-\rangle \neg \Phi$; 
$Z$ is a second-order (SO) predicate variable over sets of states, 
and $\mu Z.\Phi$ and $\nu Z.\Phi$ denote the least and greatest 
fixpoints, respectively, with $\Phi$ seen as a predicate transformer with respect to $Z$. 
By the language semantics below, one can see that the only interesting 
formulas are those that are closed wrt to SO (in addition to FO)  
variables. In fact, SO variables are needed only for technical reasons, 
 to make the fixpoint constructs available.

Given a GA $\TS = \tup{\Delta,Q,q_0,\rightarrow,\I}$,
the semantics of a $\mu\L_c$ formula $\Phi$ over $\TS$ is 
inductively defined as follows, where $v$ is an assignment 
to SO variables: 
\[
 \begin{array}{l}
\ext{\phi}[\TS][] = \set{ q \mid q\in Q \text{ and } \I(q)\models \psi } \\
\ext{\neg \Phi} = Q \setminus  \ext{\Phi} \\
\ext{\Phi_1 \land \Phi_2} = \ext{\Phi_1} \cap \ext{\Phi_2} \\	
\ext{\langle-\rangle \Phi} = \set{ q \mid \exists q', q\rightarrow q', q'\in \ext{\Phi}  } \\
\ext{Z} = v(Z) \\[0.5ex]	
\ext{\mu Z.\Phi} = \bigcap \set{\E \subseteq Q \mid \ext{\Phi}[\TS][v[Z/\E]] \subseteq \E} \\	
\ext{\nu Z.\Phi} = \bigcup \set{\E \subseteq Q \mid \E \subseteq \ext{\Phi}[\TS][v[Z/\E]] } \\	
\end{array}
\]

\noindent
A state $q\in Q$ is said to \emph{satisfy} 
a $\mu\L_c$ formula $\Phi$ 
(under a SO assignment $v$), if $q\in\ext{\Phi}$. 
We say that $\TS$ satisfies $\Phi$ if $q_0\in\ext{\Phi}$.
Observe that when $\Phi$ is closed wrt SO variables, as are 
formulas of practical interest, $v$ becomes irrelevant.
When not needed, we omit $v$ from $\ext{\cdot}$, thus using 
$\ext{\cdot}[\TS][]$.

\emph{$\mu\L_c$ model checking} is the problem of checking whether 
a GA $\TS$ satisfies a $\mu\L_c$ formula $\Phi$.
When the GA is finite, this can be solved by direct application of the 
semantics. 
Thus one can compute the set $\Win$, called the \emph{winning set}, 
of states of $\TS$ that satisfy the formula $\Phi$.
On the other hand, the problem can be shown to be 
undecidable in the general case (by reduction from the halting problem). 
In \cite{CalvaneseDMP18}, decidability is proven under sufficient conditions,
including \emph{genericity} and \emph{state-boundedness}.
We recall these notions and relevant results.

Given two FO interpretations $\I$ and $\I'$ over a set of fluents $\Fluents$ and an object 
domain $\Delta$, write $\I\sim_h\I'$ if $h:\Delta\rightarrow\Delta$ is an isormophism 
between $\I$ and $\I'$, in which case $\I$ and $\I'$ are said to be \emph{isomorphic (under $h$)}.
Intuitively, isomorphic states can be obtained by one another via object renaming.
\begin{definition}[Genericity]
	A GA $\TS = \tup{\Delta,Q,q_0,\rightarrow,\I}$ is said to be \emph{generic} if:
	for every $q_1,q_1',q_2\in Q$ and every bijection
	$h:\Delta\rightarrow\Delta$, if $\I(q_1)\sim_h\I(q_2)$,
	then there exists $q'_2\in Q$ s.t.~$q_2\rightarrow q_2'$ and $\I(q_1')\sim_h \I(q'_2)$.
	\label{def:generic}
\end{definition}
In words, a GA is generic if states identical
modulo object renaming show same behaviors, in particular, 
the have same transitions (modulo renaming).

For an interpretation $\I$, denote by $\adom(\I)$ the \emph{active domain} of 
$\I$, i.e., the set of all objects that 
occur in the interpretation of some fluent in $\I$.
\begin{definition}[State-boundedness]
	A GA $\TS = \tup{\Delta,Q,q_0,\rightarrow,\I}$ is said to be \emph{state-bounded}
	by $b\in\mathbb{N}$ if $\card{\adom(\I(q))}\leq b$, for every $q\in Q$.
	$\TS$ is said to be \emph{state-bounded} if it is state bounded by $b$, for some $b$.
\end{definition}
Intuitively, a GA is state-bounded if, in every state, 
the number of objects occurring in the interpretation of some fluent
is bounded by a given $b$.
Because $\mu\L_c$ is a strict sublanguage of the general FO variant $\mu\L$ defined in \cite{CalvaneseDMP18},
by Theorem 17 therein, we have the following result.
\begin{theorem}
	Given a generic and state-bounded GA $\TS$ and a $\mu\L_c$ formula $\Phi$, there exists 
	a finite-state GA $\finite{\TS}$ such that $\TS\models\Phi$ iff $\finite{\TS}\models\Phi$.
	\label{th:abstraction} 
\end{theorem}

\noindent
Thus, we can sidestep $\TS$ infiniteness
by checking whether $\finite{\TS}\models\Phi$,
instead of $\TS\models\Phi$.
We do not describe how to obtain $\finite{\TS}$, 
referring the reader to~\cite{CalvaneseDMP18}, 
where a procedure is provided, which requires that, in $\TS$,
\myi $\rightarrow$ is computable, and 
\myii the existence of an isomorphism between states is decidable.
The returned $\finite{\TS}$ is s.t.~$\finite{\TS} = \tup{\finite{\Delta},\finite{Q},q_0,\finite{\rightarrow},\finite{\I}}$,
with:
$\finite{\Delta}$, $\finite{Q}$, and $\finite{\rightarrow}$ suitable finite subsets of their 
$\TS$ counterparts, and $\finite{\I}$ the projection of $\I$ over $\finite{Q}$,
with $\Delta$ replaced by $\finite{\Delta}$. Notice that the  
$\TS$ and $\finite{\TS}$ share the same $\Fluents$.

$\finite{\TS}$ and $\TS$  are related by the notion of 
\emph{persistence-preserving bisimulation},
\emph{$p$-bisimulation} for short~\cite{CalvaneseDMP18}, 
i.e., a 
lifting of standard bisimulation to the case where 
states are labelled by FO (instead of propositional) interpretations.
$p$-bisimulation is defined co-inductively over triples 
$\tup{q_1,h,q_2}$, where $q_1$ and $q_2$ are states of two GAs
and $h$ is an isomorphism between their interpretations, restricted to 
the active domains. 
In details, if $\tup{q_1, h, q_2}$ is in a $p$-bisimulation $R$, then:
\begin{inparaenum}[(i)]
	\item\label{bisim:req1} $q_1$ and $q_2$ 
		have isomorphic fluent extensions, according to 
		$h:\adom(\I_1(q_1))\mapsto\adom(\I_2(q_2))$
		(objects not occurring in fluent extensions are neglected)
		--we denote this by writing $\tilde{\I_1}(q_1)\sim_h \tilde{\I_2}(q_2)$;
	\item\label{bisim:req2} for every successor $q_1'$ of $q_1$ there exists a successor 
		$q_2'$ of $q_2$ and 
		a bijection $b:\adom(\I_1(q_1))\cup\adom(\I_1(q_1'))\mapsto\adom(\I_2(q_2))\cup\adom(\I_2(q_2'))$ 
		that extends $h$ to $\adom(\I_1(q'_1))$
		s.t.~for its restriction $h'$ to $\adom(\I_2(q_2'))$, $\tup{q_1',h',q_2'}$ holds;
	\item\label{bisim:req3} the analogous of~(\ref{bisim:req2}) 
			holds for every successor $q_2'$ of $q_2$.
\end{inparaenum}

$p$-bisimilarity intuitively means that 
the identity of objects is preserved as long as they persist in the active domain
or if they have just disappeared from it. 
Two GAs are $p$-bisimilar if their respective initial states are in some 
$p$-bisimulation.  $\TS$ and $\finite{\TS}$ are $p$-bisimilar.

\subsection{Strategies}
In this paper, we consider only formulas of the form:
$$\PhiGA=\nu X.\mu Y.(\pred{turnT} \land \phi \land [-] X ) \lor{} (\pred{turnS} \land \langle-\rangle Y)$$
where $\phi$ is a FO formula over $\Fluents\setminus\set{\pred{turnS},\pred{turnT}}$.
Intuitively, $\PhiGA$ holds in all those states where: 
either \myi it is Target's turn, \myii $\phi$ holds, 
and  \myiii no matter how Target moves, 
System can reply with a sequence of moves
--which, by a slight abuse of notation, we call \emph{plan},
that takes the GA to a new state where $\PhiGA$ holds,
or \myiv it is System's turn and \myv System has a plan 
to reach a state where $\PhiGA$ holds. Notice that if the 
initial state of the GA satisfies $\Phi$, then no matter how 
Target moves (now or in its future turns), System will always 
have a plan to enforce $\PhiGA$.

Through model checking, we can obtain the winning set of $\PhiGA$. 
To this end, the following operators are needed:
 
\smallskip
\begin{compactitem}
	\item $PreE(Z) = \set{ q\in Q ~|~ \exists q\rightarrow q' \text{ s.t. } q'\in Z}$;
	\item $PreA(Z) = \set{ q\in Q ~|~ \forall q\rightarrow q' \text{ then } q'\in Z}$. 
\end{compactitem}

\smallskip
With these, we compute the approximants for the SO variable $X$, ending up with a 
greatest fixpoint. 
The initial approximant of $X$ is $X_0=Q$, and 
the next one is computed as $X_{i+1}=Y_{i}\cap X_i$, 
where $Y_i=\ext{\mu Y.(\pred{turnT}\land\phi\land [-]X)\lor (\pred{turnS} \land \langle-\rangle Y)}[\TS][v[X/X_i]]$.
$Y_i$ being a (least) fixpoint, it can be computed, as standard, 
through successive approximants $Y_{i0}=\emptyset,Y_{i1},\ldots,Y_{i_{n_i}}=Y_i$, as
$Y_{i(j+1)}= 
Y_{ij} \cup 
(\ext{\pred{turnT}\land\phi}[\TS][] \cap preA(X_i)) \cup (\ext{\pred{turnS}}[\TS][] \cap preE(Y_{ij}))$.

The winning set  $\Win$ is the resulting (greatest) fixpoint, i.e., 
$\Win=\ext{\Phi_{Sim}}[\TS][]=X_k$ (for some $k$).

If a state is in 
the winning set, System has a plan to reach a state 
where $\PhiGA$ holds. However, we do not know such plan.
We are interested not only in computing the winning states where $\PhiGA$ holds  
but also in finding a ``strategy'' showing how the System
can enforce $\PhiGA$.

Let $\TS = \tup{\Delta,Q,q_0,\rightarrow,\I}$ be a GA.
A \emph{history} of $\TS$ is a sequence $\tau=q_0\cdots q_\ell\in Q^+$
s.t., for every $i\in[0,\ell-1]$, $q_i\rightarrow q_{i+1}$. 
We denote by $\H$ the set of histories of a GA.
A System (Target) \emph{strategy} is a function 
$\strat:\H\rightarrow Q$ s.t.~if $\strat(q_0\cdots q_\ell)=q$
then $q_\ell \models \pred{turnS}$ ($q_\ell \models \pred{turnT}$) and $q_\ell\rightarrow q$.
In this paper, we are interested only in System's strategies, i.e., functions that, given a history terminating in a 
state where System moves, prescribes a legal transition, the \emph{system's move}, 
to perform next. Since $q$ does not have to be s.t.~$q\models\pred{turnT}$, System can perform move sequences.

The strategies of interest are those, called \emph{winning}, 
which enforce $\PhiGA$; these are defined next.
A history $\tau=q_0\cdots q_\ell$ is said to be \emph{induced} by 
a strategy $\strat$ if, for every $i\in[0,\ell-1]$, 
whenever $q_i\models\pred{turnS}$, 
$q_{i+1}=\strat(q_0\cdots q_i)$.
\begin{definition}[Winning strategy]
A System strategy $\strat$ is said to be \emph{winning} for a 
formula of the form $\PhiGA$
if, for every history $\tau=q_0\cdots q_\ell$ induced by $\strat$, either:
\begin{enumerate}
	\item\label{def:win-strat-req1} 
		$q_\ell\models\pred{turnS}$ and there exists a history 
		$\tau'=q_0\cdots q_\ell\cdots q_m$ induced by $\strat$ 
		s.t.~requirement~\ref{def:win-strat-req2} below 
		(with $\ell$ replaced by $m$) holds; or	
	\item\label{def:win-strat-req2} $q_\ell\models\pred{turnT}\land\phi$ and for
		all histories $\tau'=q_0\cdots q_\ell q_{\ell+1}$, 
		requirement~\ref{def:win-strat-req1} above (with $\ell$ replaced by $\ell+1$) holds.
\end{enumerate}
\end{definition}

Let $\finite{\Win}$ be the winning set of $\PhiGA$ computed on the finite-state  
GA $\finite{\TS}$ $p$-bisimilar to $\TS$. From this, we extract a 
winning strategy. Notice we cannot simply compute a strategy  
prescribing a path from a winning state to \emph{any} other winning state, as 
in the presence of loops there would be no guarantee of eventually reaching 
the goal, i.e., a state where $\pred{turnT} \land \phi$ holds.

To obtain a winning strategy, we observe that the computation of 
$\finite{\Win}$ 
amounts to computing a series of $Y$'s approximants: 
$Y_{00},\ldots,Y_{0{n_0}},\ldots,Y_{k0},\ldots,Y_{k{n_k}}$, where 
each $Y_{ij}$ corresponds to the $j$-th approximant of the set of 
states where $\pred{turnS}$ holds and from which a state in $X_{i}$ where 
$\pred{turnT}\land\phi$ 
holds can be reached through System moves only.
Thus, from $Y_{ij}$ one such state is reachable with $j-1$ System moves.
Since we are interested not in a generic $X$ approximant but 
in the winning set, we consider $Y_{k0},\ldots,Y_{kn_k}$ only, 
as these approximate the set of states that lead to
states of the winning set that satisfy $\pred{turnT}\land\phi$.
Thus, we can ``stratify'' $\finite{\Win}$ by annotating each of its states with the index $j$ of the first 
approximant $Y_{kj}$ it has appeared in. We denote the annotation of a state $q$ 
as $ann(q)=j$. In this way, to obtain a winning strategy for $\PhiGA$, it is enough to choose a transition that 
takes the current state to one annotated with a lower value.  
Thus, a winning strategy is any function $\strat$ s.t.~$\strat(q_0\ldots q_m) = q_{m+1}$ implies that $ann(q_m)>ann(q_{m+1})$. In fact, $\strat$ is \emph{memoryless}, in that
it does not depend on any state of the input history but the last one ($q_m$).

Next, we describe how one such strategy $\finite{\strat}$ computed on $\finite{\TS}$ 
for a goal $\PhiGA$ can be actually executed on $\TS$. 

\begin{definition}[$p$-bisimilar strategy transformation]\label{def:p-bisim-transform}
Consider two $p$-bisimilar GAs $\TS$ and $\TS'$
and let $\rho$ be a System strategy for $\TS$.
A strategy $\rho'$ for $\TS$ is said to be a 
\emph{$p$-bisimilar transformation of $\rho$ to $\TS'$}, if 
there exists a $p$-bisimulation $R$
s.t.~for every history $\tau=q_0\cdots q_\ell$ of $\TS$ induced by $\rho$, 
there exists a history $\tau'=q'_0\cdots q'_\ell$ of 
$\TS'$ induced by $\rho'$
and a sequence of bijections $h_i:\adom(\I(q_i))\rightarrow\adom(\I(q'_i))$ $(i=0,\ldots,\ell)$,
s.t., for every $i\in[0,\ell]$, 
\myi $\tup{q_i,h_i,q'_i}\in R$ and 
\myii if 
$\tilde{\I}(q_i)\sim_{h_i} \tilde{\I'}(q'_i)$ and 
$\tilde{\I}(q_{i+1})\sim_{h_{i+1}} \tilde{\I'}(q'_{i+1})$
then there exists a bijection 
$b:\adom(\I(q_i))\cup \adom(\I(q_{i+1}))\rightarrow\adom(\I(q'_i))\cup\adom(\I(q'_{i+1}))$
s.t.~$h_i=b\mid_{\adom(\I(q_i))}$ and $h_{i+1}=b\mid_{\adom(\I(q_{i+1}))}$.
\end{definition}

\begin{theorem}\label{th:strat-transform}
If two GAs $\TS$ and $\TS'$ are $p$-bisimilar then there exists 
a System strategy $\strat$ on $\TS$ iff there exists
a System strategy $\strat'$ on $\TS'$ that is a $p$-bisimilar transformation of $\strat$.
\label{th:transformation}
\end{theorem}
\begin{proof}
By $p$-bisimilarity, there exists a bisimulation $R$ 
s.t.~for every history $\tau=q_0\cdots q_\ell$ of $\TS$ induced by $\strat$,
there exists a history 
$\tau'=q'_0\cdots q'_\ell$ of $\TS'$ that fulfills the requirement 
of $\tau'$ in Def.~\ref{def:p-bisim-transform}. 
For the if-part, we define $\strat'$ as $\strat'(q'_0\cdots q'_{\ell-1})=q_\ell'$, 
for every history $q_0\cdots q_{\ell-1} q_\ell$ of $\TS$ induced by $\strat$, 
s.t.~$q_{\ell-1}\models\pred{turnS}$.
Only-if part is analogous.
\end{proof}

Theorem~\ref{th:strat-transform} provides us with a constructive way to 
transform a strategy executable on $\TS$ into one on $\TS'$. 
It essentially requires to transform the states of a history $\tau$ of $\TS$ 
into those of a history $\tau'$ of $\TS'$, by applying the isomorphisms
that, in $R$, associate the states of the two GAs, while preserving
the identity of the objects that persist and of those that have just 
disappeared from the active domain. 

For an example of this, assume to have computed a winning strategy on 
$\finite{\TS}$ and to want to execute a transformation of it on $\TS$.
Assume that $\TS$ has traversed the history $\tau=q_0\cdots q_\ell$.
From this, we obtain the corresponding $\finite{\TS}$ history 
$\finite{\tau}=\finite{q}_0\cdots \finite{q}_\ell$, compute the move 
$\finite{q}=\finite{\strat}(\finite{\tau})$, and then translate it back to 
a move for $\TS$, according to any isomorphism chosen as described in 
the theorem above.


\subsection{Controller synthesis}
In this section, we show how we can exploit $\mu\L_c$ model checking
to both compute a simulation between the target program and the available 
program, and to synthesize the corresponding controller responsible for orchestrating the available system.
Given the two programs $\progTinit$ and $\progSinit$, together with the corresponding theories $\BATT$ and $\BATS$, 
we now construct a GS $\TS$  \emph{induced by} these programs, i.e., that captures the crossproduct of their execution.

It is the GA $\TS = \tup{\Delta_{\TS},Q,q_0,\rightarrow,\I}$ built as follows:

\smallskip
\noindent
\textbf{Object domain}: $\Delta_{\TS}=\Delta\cup\set{\turnT,\turnS}\cup\Gamma_{\progTinit}\cup\Gamma_{\progSinit}\cup\Gamma_{\pedixS}$, where the latter set is the union of the syntactic closures of all programs $\prog_{\Act}(\xx)$ such that a mapping rule $\Act(\xx) \leftrightarrow \prog_{\Act}(\xx)$ exists;

\noindent
\smallskip
\textbf{States}: $Q \subseteq \set{\turnT,\turnS} \times \Confs{\progTinit}\times\Confs{\progSinit} \times \Confs{\bar{\progS}}$ is the set of states, where $\Confs{\bar{\progS}}$ is a special set of configurations, discussed below. Each $q = \tup{\turn,\tup{\progT,\xxT,\sitT},\tup{\progS,\xxS,\sitS},\tup{\bar{\progS},\bar{\xxS},\bar{\sitS}}}$ is such that
	$\turn\in\set{\turnT,\turnS}$ specifies the \emph{turn}, i.e., which program, target ($\turnT$) or available ($\turnS$), moves next.  
	$\tup{\progT,\xxT,\sitT}$ and $\tup{\progS,\xxS,\sitS}$ are, respectively, configurations of the target and available programs. Specifically, $\progT$ is the remaining fragment of the target program to execute, while $\progS$ is the program representing all remaining possible executions of the available system; moreover, $\xxT$ and $\xxS$ are the current environments for these programs, 
representing the current assignments of their pick variables, and finally $\sitT$ and $\sitS$ are the situations of $\BATT$ and $\BATS$ resulting, respectively, from the portion of $\progTinit$ and $\progSinit$ executed so far. 
Finally, the additional configuration $\tup{\bar{\progS},\bar{\xxS},\bar{\sitS}}\in \Gamma_{\pedixS} \times \Delta^{k_\pedixS} \times S$ represents the remaining fragment of the program which corresponds, through a mapping rule, to the last target action $\Act(\xx)$ being realised by the system: after each turn of the target program, this is precisely $\prog_{\Act}(\xx)$. As defined in the definition of the transition relation, its role is to make sure that the evolution of the system, represented by the sequence of evolving configurations in $\Confs{\progSinit}$, corresponds to a complete execution of $\prog_{\Act}(\xx)$.

\smallskip
\noindent
\textbf{Initial state}: $q_0 = \langle \turnT,\tup{\progTinit,\xxT^0,S^0_\pedixT}$, $\tup{\progSinit,\xxS^0,S^0_\pedixT}$, $\tup{ \True?,\xxS^0, S^0_\pedixS}\rangle$. Initially, it is the turn of the target program; the target program is $\progTinit$ with 
	initial assignment $\xxT^0$; the available program is $\progSinit$ with 
	initial assignment $\xxS^0$; and the two BATs are in their initial situations. Also, there is yet no target action to be replicated, and therefore no associated program -- $\True?$ is the empty program;

\smallskip
\noindent
\textbf{Transitions}: $\rightarrow \subseteq Q \times Q$ is the transition relation, s.t.~$Q$ and $\rightarrow$ are defined through mutual induction: $q_0\in Q$ and if $q\in Q$ we have that $q'\in Q$ for all $q\rightarrow q'$. 
A transition $\langle \turn,\tup{\progT,\xxT,\sitT}$,$\tup{\progS,\xxS,\sitS}$,$\tup{\bar{\progS},\bar{\xxS},\bar{\sitS}} \rangle$ $\rightarrow$ $\tup{\turn',\tup{\progT',\xxT',\sitT'},\tup{\progS',\xxS',\sitS'},\tup{\bar{\progS}',\bar{\xxS}',\bar{\sitS}'}}$ exists iff either it is the turn of the target program and a possible next target situation $\sitT' = \do(\Act(\xx),\sitT)$ is selected (that is, resulting from the execution of the action $\Act(\xx)$), or it is the turn of the available concurrent program and the target situation $\sitT'$ can be replicated by executing $\prog_{\Act}(\xx)$ in the system. We also need to make sure that the resulting situation $\sitS'$ is a situation in which the program $\prog_{\Act}$ corresponding to $\Act$ is final (for this, we use $\Do$). Since $\prog_{\Act}(\xx)$ is single-step, we use $\turn$ to establish a \emph{strict} alternation between $\progTinit$ and $\progSinit$.
Therefore, $\langle \turn,\tup{\progT,\xxT,\sitT}$,$\tup{\progS,\xxS,\sitS}$,$\tup{\bar{\progS},\bar{\xxS},\bar{\sitS}} \rangle$ $\rightarrow$ $\tup{\turn',\tup{\progT',\xxT',\sitT'},\tup{\progS',\xxS',\sitS'},\tup{\bar{\progS}',\bar{\xxS}',\bar{\sitS}'}}$ iff either:
\begin{itemize}
\item $\turn=\turnT \land \TransObs(\progT[\xxT],\sitT,\sitS,\progT'[\xxT'],\sitT') \land \exists \Act,\xx. ~\sitT'=\do(\Act(\xx),\sitT) \land \bar{\progS}' = \prog_{\Act} \land  \progS'=\progS \land \xxS'=\xxS \land \sitS'=\sitS \land \bar{\xxS}' = \bar{\xxS}[\xx] \land  \bar{\sitS}'=\bar{\sitS} \land \turn'=\turnS$;
\item $\turn=\turnS \land \Trans(\progS[\xxS],\sitS,\progS'[\xxS'],\sitS') \land \Trans(\bar{\progS}[\bar{\xxS}]$, $\bar{\sitS}, \bar{\progS}'[\bar{\xxS'}],\bar{\sitS}') \land \progT'=\progT \land \xxT'=\xxT \land \sitT'=\sitT  \land ( \Final(\bar{\progS}'[\bar{\xxS}'],\sitS') \land \turn'=\turnT  \lor \turn'=\turnS ).$
\end{itemize}

The former case applies when it is the turn of the target program and a possible next target situation $\sitT' = \do(\Act(\xx),\sitT)$ is selected (corresponding to an action $\Act(\xx)$). The target configuration is progressed, while the system remains idle. The last configuration registers the program to execute, with free variables replaced by $\xx$ (denoted here by  $\bar{\xxS}[\xx]$) . The latter case applies when it is the turn of the available program, while the system remains idle. Note that we also progress the program corresponding to the last configuration $\bar{\confS}$, testing that the program remaining fragment of $\prog_{\Act}(\xx)$ may terminate, and in this case we allow the turn to be ``given back'' to the target program.  

\smallskip
\noindent
\textbf{Labelling}: 
$\I:Q\rightarrow \Int^{\Fluents}_{\Delta_{\TS}}$ where  
$\Fluents=\Fluents[T]\cup\Fluents[S]\cup\set{\pred{turn},\pred{progT},\pred{progS},\pred{finalT},\pred{finalS},
\pred{envT},\pred{envS}}$, for fluents in $\Fluents[T]$ and $\Fluents[S]$ with the situation argument suppressed.
Informally, we ``make visible'' the internal structure of the state through the labelling $\I$, so that we can evaluate $\mu\L_c$ formulas on $\TS$. Formally, $\I$ is as follows:

\myi Fluents in $\Fluents[T]$ and $\Fluents[S]$ are interpreted according to 
the interpretation provided by the model $M$ of $\BATT\cup\BATS$ (and the Situation Calculus and ConGolog axioms) at situations $\sitT$ and $\sitS$\footnote{$M$ is unique since the initial situation is fully specified.}. Specifically, for every $q\in Q$ and $\xx\in\vec{\Delta}$, and for every $\fluent\in\Fluents[S]$ 
we have $\fluent^{\I(q)}(\xx)$ iff $\fluent^{M}(\xx,\sitS)$. Similarly, for each fluent $\fluent\in\Fluents[T]$, 
	\begin{compactitem}
	\item if $\turn=\turnS$ then $\fluent^{\I(q)}(\xx)$ iff $\fluent^{M}(\xx,\sitT)$;
	\item if $\turn=\turnT$ then:
		\begin{compactitem}
		\item if $\neg\Obs(\fluent)$ then $\fluent^{\I(q)}(\xx)$ iff $\fluent^{M}(\xx,\sitT)$;
		\item otherwise $\fluent^{\I(q)}(\xx)$ iff $\varphi_{\fluent}^{M}(\xx,\sitS)$.
		\end{compactitem}
	\end{compactitem}
	Intuitively, we define the labelling function $\I$ based on the model $M$, making sure that, whenever $\turn=\turnT$, $\I$ is consistent with the mappings between $\BATT$ and $\BATS$ for all fluents representing observations of the available system.
	
\myii For the remaining fluents in $\Fluents$, assuming $q=\tup{\turn,\tup{\progT,\xxT,\sitT},\tup{\progS,\xxS,\sitS}}$, we have: $\pred{turn}^{\I(q)}=\set{\turn}$, $\pred{progT}^{\I(q)}=\set{\progT}$, $\pred{progS}^{\I(q)}=\set{\progS}$, $\pred{envT}^{\I(q)}=\set{\xxS}$, $\pred{envS}^{\I(q)}=\set{\xxS}$.
	Finally, the $0$-ary predicate $\pred{finalT}$ is true iff $\Final(\progT[\xxT],\sitT)$, and analogously for $\pred{finalS}$. 

\medskip
The GA above is essentially the tree of
executable (combinations of) configurations for $\progTinit$ and $\progSinit$, with the state-labelling providing an interpretation of the set of fluents $\Fluents$, used to verify $\mu\L_c$ formulas on $\TS$.  
Notice that labelings retain all the relevant information about states.

Satisfaction of the following $\mu\L_c$ by $\TS$ implies 
the existence of a simulation between $\progTinit$ and $\progSinit$: 

\[
\Phi_{Sim} = \nu X.\mu Y.( (\phi_{\textsc{ok}} \land [-] X ) \lor{} (\pred{turnS} \land \langle-\rangle Y))
\]

\noindent
where $\phi_{\textsc{ok}}=(\pred{finalT} \mapsto \pred{finalS}) \land \pred{turnT}$. 
Intutively, $\phi_{\textsc{ok}}$ holds in those states in which it is the turn of the target and if the target program may terminate so can the available program. Therefore, the formula requires that no matter how the target program evolves to a new program $\progT'$ through the execution of an action $\Act(\xx)$ from a state in which $\phi_{\textsc{ok}}$ holds, from that successor state (where $\pred{turnS}$ holds) there exists a \emph{sequence} of transitions corresponding to a complete execution of $\prog_{\Act}(\xx)$, and from where the whole property still holds. 

\begin{theorem}
Let $\phiWin$ be the set of winning states in $\TS$ wrt $\Phi_{Sim}$. 
Then, $\progTinit$ is realizable by $\progSinit$ iff $q_0\in \phiWin$.
\label{th:realizable_infinite}
\end{theorem}

\begin{proof} (Sketch.) 
It follows from the fact that for each state $\tup{\turnT,\tup{\progT,\xxT,\sitT},\tup{\progS,\xxS,\sitS},\bar{\confS}}\in \phiWin$ we have $(\progT[\xxT],\sitT)\preceq (\progS[\xxS],\sitS)$. It is immediate to see that if a state $q\in\phiWin$ then $q\models\phi_{\textsc{ok}}$ and that all fluent mappings are satisfied by definition of $\I$ in $\TS$. Also, since any possible target action $\Act(\xx)$ is captured by a successor of such $q$, it follows that there exists a path in $\TS$ corresponding to $\prog_{\Act}(\xx)$, and that the resulting state is in $\Win$. With analogous reasoning, one can see that opposite holds as well, as it implies $(\progTinit[\xxT^0],S_\pedixT^0)\preceq (\progSinit[\xxS^0],S_\pedixS^0)$.
\end{proof}


Although the fixpoint computation described earlier gives us a way of capturing the winning set $\phiWin$, the number of approximants that we need to compute is bounded (by the size of the GA) only if $\TS$ is finite. Since, in our case, $\TS$ can be infinite, the fixpoint cannot be computed, in general.

\begin{lemma}
$\TS$ as above is generic.
\end{lemma}

\begin{proof}(Sketch)
The result is a direct consequence of the fact that $\TS$ is defined 
on two BATs and the transition relation is defined by 
a FO specification involving only the current state and the next one.  
\end{proof}

Thus, by Theorem~\ref{th:abstraction}, if $\TS$ is state-bounded, 
there exists a finite GA $\finite{\TS}$ which we can use to verify $\phiWin$
(instead of using the infinite-state $\TS$). 
Importantly, $\finite{\TS}$ is effectively computable.

\begin{lemma}
If $\TS$ is state-bounded, then $\finite{\TS}$ is effectively computable.
\label{lemma:computable}
\end{lemma}
\begin{proof}(Sketch)
Consequence of Theorem~17 of~\cite{CalvaneseDMP18} and the fact that, for a given state, if 
$\TS$ is state-bounded,  the successor state is computable.
\end{proof}

By applying our general technique for $\PhiGA$ formulas, we can  compute a winning strategy from $\phiWin$. While this strategy is not directly executable on $\TS$,
we can exploit the notion of strategy transformation introduced earlier.
Therefore, a concrete inductive procedure for \emph{executing} the controller corresponding to the strategy for $\TS$ is provided in the proof of Theorem~\ref{th:transformation}, which we can  directly execute. 

Finally, as a winning strategy $\strat$ as above is given, we can directly obtain the corresponding controller for $\progTinit$ and $\progSinit$, as explained in Definition~\ref{def:controller}, as follows. For every history $\tau=q_0 \cdots q_\ell q_{\ell+1}$, with $q_\ell\models\pred{turnT}$ (and thus $q_{\ell+1}\models\pred{turnS}$), for $q_i=\tup{\turn_i,{\confT}_i,{\confS}_i,\bar{\confS}_i}$, we return the sequence $\rho(\tup{{\confT}_\ell,{\confS}_\ell,{\confT}_{\ell+1}})=\confS^0\cdots \confS^m$ with $\confS^i$ is the system configuration of each state $q^i=\strat(\tau q^0\cdots q^i)$, for $i\in[0,m-1]$. 

\section{Conclusions}


In this paper, by exploiting recent results on the Situation Calculus, we
have been able to effectively synthesize controllers for 
manufacturing-as-a-service scenarios, under the assumption of state-boundedness.
However, we have only scratched the surface of what KR formalisms like the
Situation Calculus can bring to this new manufacturing paradigm.  
For instance, it would be of interest to
equip resources with autonomous deliberation capabilities \cite{Baldwin:89a}, e.g., to react to
exogenous events during execution, or to monitor streaming production data \cite{Lee//:15a}, 
or consider the explicit treatment of time and other continuous value quantities \cite{Behandish//:18a}. 
We leave these to further work.
\clearpage

\bibliography{synthesis,manufacturing}
\bibliographystyle{aaai}

\end{document}